\documentclass[preprint]{article} 
\usepackage{algorithm}
\usepackage{algorithmic}

\usepackage{booktabs}       
\usepackage{amsfonts}       
\usepackage{amssymb}
\usepackage{nicefrac}       
\usepackage{microtype}      
\usepackage{amsthm}
\usepackage{paralist,subcaption}
\usepackage{mathtools}
\usepackage{fancyvrb}
\usepackage{pgfplots}
\pgfplotsset{compat=1.16}
\newtheorem{theorem}{Theorem}
\newtheorem{lemma}[theorem]{Lemma}

\newtheorem{definition}{Definition}

\usetikzlibrary{arrows, quotes,backgrounds,calc}
\usepackage{tikz}
\usetikzlibrary{arrows.meta}
\usetikzlibrary{arrows,shapes,decorations,automata,backgrounds,petri}
\tikzset{point/.style = {fill=black,circle,inner sep=0.7pt}}
\tikzset{
  graph vertex/.style={
    circle,
    draw,
  },
    graph vertexx/.style={
    draw,
  },
  graph directed edge/.style={
    ->,
    >=stealth,
    thick,
  },
  graph tree edge/.style={
    graph directed edge
  },
  graph forward edge/.style={
    graph directed edge,
    every edge/.style={
      edge node={node [fill=white,font=\scriptsize] {f}},
      loosely dotted,
      draw,
    },
  },
  graph back edge/.style={
    graph directed edge,
    every edge/.style={
      edge node={node [fill=white,font=\scriptsize] {b}},
      densely dotted,
      draw,
    },
  },
  graph cross edge/.style={
    graph directed edge,
    every edge/.style={
      edge node={node [fill=white,font=\scriptsize] {c}},
      dotted,
      draw,
    },
  },
}

\title{Minimal Conditions for Beneficial Neighbourhood Search and Local Descent}

\author{%
  Mark G. Wallace\\
  Dept. of Data Science and Artificial Intelligence\\
  Faculty of Computer Science\\
  Monash University\\
  Clayton, Vic 3168\\
  Australia\\
  \texttt{mark.wallace@monash.edu} \\
}

\begin{document}
\maketitle

\begin{abstract}
This paper investigates what properties a neighbourhood requires to support beneficial local search. 
We show that neighbourhood locality, and a reduction in cost probability towards the optimum, support a proof that  
search among neighbours is more likely to find an improving solution in a single search step than blind search.
This is the first paper to introduce such a proof.
The concepts underlying these properties are illustrated on a satisfiability problem class, and on travelling salesman problems.
Secondly, for a given cost target $t$, we investigate a combination of blind search and local descent termed {\em local blind descent}, and present various conditions under which the expected number of steps to reach a cost better than $t$ using local blind descent, is proven to be smaller than with blind search.
Experiments indicate that local blind descent, given target cost $t$,  should switch to local descent at a starting cost that reduces as $t$ approaches the optimum.  

\end{abstract}

\section{Introduction}
\subsection{Neighbourhood search and local descent}
There is a wide variety of techniques for tackling large scale combinatorial optimisation problems.  
Incomplete search methods are typically used to achieve the required scalability.
Indeed 
\cite{co-2001}
write: ``Most `general-purpose' optimization techniques rely on 
some sort of hill climbing
at the lowest level. 
These
techniques include golden section search, Brent’s method, the downhill simplex method, direction-set methods, conjugate gradient methods, quasi-Newton methods, simulated annealing,  
evolution strategies, evolutionary programming and various types of hill climbers themselves''. 
All these techniques involve a sub-algorithm where a current solution, or set of solutions, are modified in some way to produce new candidate solutions.
We call this neighbourhood search.

We examine the benefit of neighbourhood search 
and therefore why these techniques deploy them.

In our analysis of neighbourhood search, we make the conservative assumption that 
the search for an improving neighbour uniformly at random selects neighbours to evaluate until an improving neighbour is found.

Blind search selects candidate solutions from the search space as a whole, uniformly at random, evaluating the cost of each one until a solution with a desired cost (or better) has been found. 

We say that neighbourhood search performs better than blind search if the probability the next candidate solution has better cost than the current solution is higher than the probability a better solution is selected by blind search.\footnote{Many search algorithms pick neighbours using heuristics rather than picking neighbours blindly: if it can be proven that the heuristic improves the probability of selecting an improving neighbour, then of course our proof that local search is beneficial carries over to this heuristic.}
This is the first paper to show  minimal general conditions on neighbourhoods under which neighbourhood search is expected to outperform blind search.
The result is proven to hold for a range of starting cost levels.

When when neighbourhood search finds an improving solution, it becomes the current point, and \emph{local descent} continues from there.
If there is no such point in the neighbourhood, then local descent has reached either a local optimum or a plateau.

Modern heuristic and metaheuristic methods \cite{alorf2023} include ways to avoid or escape from plateaux and local optima.
However, this paper focusses on the progress made by local descent towards a target cost level, and does not introduce new methods to escape it.

However, even if neighbourhood search had higher probability of improvement, the expected amount of improvement with blind search can be greater than with neighbourhood search. 
Indeed this is illustrated with an example on page \pageref{sec:betterblind}.

This paper addresses local descent in two contexts.
In the first context the current point is the best found so far.  In this context every improving neighbour yields a new best solution.
The question addressed is how fast local descent is expected to improve on the current solution.

In the second context, the current point has a poorer cost than the cost $t$ of the best point found previously.  
In this context an improving neighbour may still have a cost no better than $t$. 
The question to be addressed in the second context is how quickly local descent is expected to find a point with cost better than $t$.

We say local descent is beneficial if, under the chosen measure, local search is better than blind search.
This paper investigates the conditions under which local descent is beneficial in each of the two contexts given above.
In particular, a contribution of this paper is to establish specific conditions under which a 
na\"{\i}ve form of local descent is proven to be beneficial.

\subsection{Structure of the paper}
The related work section 
introduces some previous research related to this paper; 
the next section introduces the properties needed for neighbourhood search to be beneficial,including 
Neighbours Similar Cost (NSC) and gives the formal underpinnings of our proofs;
the section titled "Probability of Improvement"
gives theorems and some proofs that neighbourhood search is beneficial;
the next two sections give theoretical and practical examples (2-SAT and TSP) illustrating and investigating NSC and other properties; 
 the section "Rate of improvement" investigates the rate of improvement with local descent and blind search; the section "Local blind search" analyses the expected number of steps to reach a target cost, assuming NSC;
the final section concludes.

\section{Related Work}
\label{sec:litrev}
\subsection{Cost Function}
For ``blackbox'' optimisation problems, the objective function is unknown, unexploitable or non-existent \cite{blackbox2021}.
This has the same implications as the "No Free Lunch" (NFL) theorems \cite{nofreelunch}, where the objective function is an unknown member of a set of functions, and is revealed only by evaluating its value point by point.
In this paper we call the value of the objective function, applied to a point in a problem's search space, the \emph{cost} of the point.

We assume that, as for blackbox objective functions, the cost of a point cannot be predicted without selecting and evaluating it.
We also assume there is no access to an improving ``direction'' in which neighbours would tend to have better cost than the current solution.
However we admit neighbourhoods with a locality property, discussed in the next section, thereby escaping from the conditions and conclusions of NFL.

\subsection{Locality}
Consider the cost values of neighbours of a current solution. Supposing their distribution is the same as the distribution of cost values in the problem search space as a whole.
In this case neighbourhood search cannot outperform blind search.

To be precise, suppose, for any given current solution 
the probability that a neighbour of that solution has 
any given cost $k$ 
is the same as the probability that any point in the search space has cost $k$.
In this case, 
searching uniformly at random in the neighbourhood of any previous solution is no more likely to yield a solution with cost better than the current 
solution
than evaluating a point selected randomly from the whole search space. 

Hence, the key general condition under which neighbourhood search is beneficial is a locality property. 
Essentially, a candidate solution chosen from the neighbourhood of the current solution is more likely to have a cost similar to the current solution than a candidate solution chosen from the search space as a whole \cite{McDermott_2020}.
This alone is sufficient to escape from the negative conclusions of the NFL theorems \cite{streeter2003}. 

Locality is a property of well-known neighbourhoods used in solving many combinatorial problems.  For example in a maximum satisfiability problem flipping the truth value of a single variable can change only the truth of those clauses in which the variable appears; in a travelling salesman problem a 2-opt changes only the cost of 2 links; and in a graph partitioning problem the swap changes only the edges attached to the swapped node.
In each case only a few of the terms in the objective function are affected, so that the new cost tends to be similar to the previous cost.

Locality is arguably the simplest \emph{useful} condition that a neighbourhood can be constructed to satisfy in a combinatorial problem.
Variants of locality have been introduced in the literature, based on the average cost of neighbours, the maximum difference between neighbours, and the change in cost along paths where successive points are neighbours.

Many studies have focussed on the expected cost value of the neighbours of any point with a given cost.
\cite{grover1992} analysed the average difference between a candidate solution and its neighbours, for five well-known combinatorial optimisation problems. 
Since this difference is positive for candidates with less than average cost it implies that any local optima must have a better than average cost.
These ideas were generalised to ``elementary landscapes'' \cite{elem-landscape-2008}.
If $x$ is an arbitrary element of the problem search space  and $y$ is drawn uniformly at random from the neighbours of $x$, then the expected cost of $y$ is a fixed fraction of the distance between $f(x)$ and $\bar{f}$ - the mean value of $f(s):s \in S$:
$$E[f(y)] = f(x) + (k/d)(\bar{f} - f(x))$$
This supports conclusions about the expected value of neighbours as the search space scales up, and about plateaus in the landscape.
However it does not yield the probability that a neighbour has better cost than the current point, or support conclusions about the benefits of neighbourhood search. 

For each cost difference we constrain the probability that two neighbours differ by this amount.  Specifically we define the property of neighbours' similar cost (NSC) in terms of the increase in probability that neighbours have a cost difference $\delta$ over the probability an arbitrary pair of points in the search space have this cost difference.

Given a measure of distance between points in a problem search space, the $L$-Lipschitz condition imposes the following condition between any pair of points, $x$ and $y$:
$$|f(x) - f(y)| \leq L \times |x-y|$$
\cite{lipschitz-2005}.
Given this condition optimisation results can be proven for various forms of direct search \cite{direct-search-2003}, section 3.  
In this paper we do not have a measure of distance between search space points: only between their cost values.
However a condition on the maximum cost difference between neighbours can be expressed as a form of $L$-Lipschitz condition 
Applying this condition to neighbouring points in a search space, it imposes that given a cost difference $L$, for every point $x$ and neighbour $y$ of $x$:
$$|f(x) - f(y)| \leq L$$
We will investigate the impact of this condition in some benchmark problems on pages \pageref{sec:betterblind} and \pageref{sec:descexp} below.

Many researchers have explored conditions on points connected by a path in which successive points are neighbours.  These include basins and funnels \cite{basins-funnels-2022}, fitness distance \cite{fitness-distance-1995}, auto-correlation \cite{autocorrelation-1990}, niching \cite{niching-1994} among others.
Search algorithms successfully exploiting these properties of landscapes have been investigated, for example subthreshold-seeking local search \cite{subthreshold-2006},
which exploits the number of basins of attraction in a landscape.

For the results of this paper, \emph{no} conditions on paths or on global properties of landscapes such as the number of local optima, are required.

Instead  we define the property of Neighbours' Similar Cost (NSC) in terms of the increase in probability that neighbours have a small cost difference $\delta$ over the probability an arbitrary pair of points in the search space have this cost difference.

\subsection{Local descent}

Over 30 years ago, Johnson et.al. asked ``How easy is neighbourhood search?'' \cite{easy-1988}.  They investigated the complexity of finding locally optimal solutions to NP-hard combinatorial optimisation problems.  They show that, even if finding an improving neighbour (or proving there isn't one) takes polynomial time, finding a  local optimum can take an exponential number of steps.

\cite{tovey1985} considers hill-climbing using flips of $n$ zero-one variables.  If the objective values are randomly generated the number of local optima tends to grow exponentially with $n$, and the expected number of successful flips to reach a local optimum from an arbitrary point grows linearly with $n$.  A more general analysis of neighbourhood search in \cite{tovey-3} explores a variety of algorithms to reach a local optimum.
More recently \cite{cohen2020steepest} showed that this still holds even if the objective is a sum of terms, each comprising no more than seven variables. 
However the number of steps to reach a local optimum does not enable us to infer the number of steps to reach a given cost level.

To compare local descent with blind search we use the expected number of steps to reach a given level of cost.
A step (in both local descent and blind search) is the selection and evaluation of a single point.
This paper is the first to give conditions under which 
the expected number of steps using a na\"{\i}ve version of local descent is lower than under blind search.

\section{Properties needed for neighbourhood search to be beneficial}
\label{sec:nscge}
In the following, for uniformity, we assume that optimisation is cost minimisation. 
We assume a finite range of integer cost values, and without loss of generality, we set the optimum cost $k_{opt}$ to be $0$.

\subsection{Definitions}
The probability a point has a given cost is its cost probability.
The cost probability is directly related to the concept of the \emph{density of states} which applies to continuous cost measures encountered in solid state physics \cite{fitness_distribution}.
That work additionally shows how to estimate the density of states for a problem using Boltzmann strategies.

The expression {\em neighbour of cost $k$} means 
a member of the set of neighbours of points  with cost $k$.  These neighbours typically have costs close to $k$, if the neighbourhoods have the $\mathrm{NSC(k)}$ property.
The neighbour's cost probability is the probability that a neighbour of cost $k_1$ has a given cost $k_2$.

\subsubsection{Neighbourhood search symbols}
\label{local_search_intro}
We introduce the following definitions:
\begin{itemize}
\item  The \emph{cost range} is the set of integers, $K = k_{opt} \ldots k_{max}$, where $k_{opt}=0$ is the optimal cost (or cost), and $k_{max}$ the worst.
\item $p(k)$ is the probability a point has cost $k \in K$.
We extend the range of $p$ by writing $\forall {\rm \ integers \ } k \notin K : p(k)=0$
\item The probability that a neighbour of a cost $k_1 \in K$ has cost $k_2$ is $pn(k_1,k_2)$.
If $k_2 \notin K$ then $pn(k_1,k_2)=0$.
\item If $\delta>0$ then $p(k \pm \delta) = p(k+\delta)+p(k-\delta)$.  $p(k \pm 0) = p(k)$.
Similarly for $pn(k, k \pm \delta)$.

\item $p^<(k)$ is the probability blind search selects a point better than $k$.  
We call it the blind probability of improving:
$$p^<(k) = \sum_{\delta=1}^k p(k-\delta)$$

\item $pn^<(k)$ is the probability neighbourhood search, starting from cost $k$, selects a point better than $k$. 
We call it the neighbourhood probability of improving:
$$pn^<(k) = \sum_{\delta=1}^k pn(k,k-\delta)$$
\end{itemize}

\subsection{Neighbourhood Weight}
\label{sec:nsf}
We have already given examples of neighbourhoods designed for local descent, such a 2-opt, 
where neighbours have similar cost.
One way to formalise the property that neighbours have similar cost, would be ``the probability neighbours have a cost difference of $\delta$ increases with decreasing $\delta$''.
Unfortunately - even where neighbours have similar cost - if the search space has very few points with cost near the optimum $0$, then the probability a neighbour of an optimal point has cost $0 + \delta$ might not increase with decreasing $\delta$.
More generally, despite neighbours having similar cost, if $k_1$ is near the optimum then $pn(k_1, k_1 \pm \delta)$ (the probability a neighbour of a point with cost $k_1$ differs from $k_1$ by $\delta$) might also not increase with decreasing $\delta$.


Consequently, in order to formalise the neighbours similar cost property $\mathrm{NSC(k)}$, we use the \emph{increased} probability a neighbour of a point with cost $k$ has cost $k \pm \delta$ over the cost probability $p(k \pm \delta)$ that an arbitrary point in the search space has cost $k - \delta$ or cost $k + \delta$.

We introduce the function $r(k,\delta)$
which is a weighting associated with cost distance $\delta$ for points in the neighbourhood of any point with cost $k$.
Accordingly
$pn(k,k \pm \delta) = p(k \pm \delta) \times r(k,\delta)$.
We call \emph{r} the {\em NWeight}.
\begin{definition}[NWeight]
The NWeight $r(k,\delta)$ for cost $k$ and cost difference $\delta$ gives probability a neighbour of a point with cost $k$ has cost $k \pm \delta$:
$$ pn(k,k \pm \delta) = p(k \pm \delta) \times r(k,\delta) $$
\end{definition}
Clearly, summing all the disjoint probabilities for a given $k$:
$$\sum_{\delta \in K} pn(k \pm \delta) = \sum_{\delta \in K} (r(k,\delta) \times p(k \pm \delta)) = 1$$
The Neighbourhood Similar Cost $\mathrm{NSC(k)}$ property holds if the NWeight $r(k,\delta)$ increases as $\delta$ decreases.

We made the point earlier that there is no access to an improving ``direction'' in which neighbours would tend to have better cost than the current solution.
Consequently the probability that any neighbour with cost  $c \in \{k+\delta, k-\delta\}$ has cost $k-\delta$ is no different from the probability any point with cost $c \in \{k + \delta, k-\delta\}$ has cost $k-\delta$.
Specifically.
for each cost level $k \in K$, for each cost difference $\delta \in K {\rm \ where\ } p(k \pm \delta) >0$,
$$ \frac{pn(k, k-\delta)}{pn(k,k \pm \delta)} \geq \frac{p(k-\delta)}{p(k \pm \delta)} $$
Using the NWeight $r(k,\delta)$ this is equivalent to the condition
\begin{definition}
The neighbourhood of $k$ is unbiased if
\begin{equation*}
\forall \delta : pn(k, k-\delta) \geq r(k,\delta) \times p(k-\delta) 
\end{equation*}
\end{definition}
We say the neighbourhood of $k$ is positively biased if $pn(k, k-\delta) > r(k,\delta) \times p(k-\delta)$.

The Neighbourhood Similar Cost $\mathrm{NSC(k)}$ property 
introduced in definition \ref{eq:nsf} is a property of a \emph{cost level} rather than a point in the search space.
At a given cost level, some points may have improving neighbours, while other might be local optima.  Unless it is globally optimal, a locally optimal point does not have the same proportion of better and worse neighbours as there are in the search space as a whole. Only on average, over all points at one cost level, does NSC require the proportion of improving neighbours to be as good or better than in the whole search space.
$\mathrm{NSC(k)}$ is defined as follows.
\begin{definition}
$NSC(k)$ holds if the neighbourhood of $k$ is unbiased, and the NWeight $r(k,\delta)$ increases as $\delta$ decreases: \\
\begin{alignat*}{2}
&\forall \delta : pn(k,k-\delta) \ge r(k,\delta) \times p(k-\delta) &\\
&\forall \delta_1 \leq \delta_2 : r(k,\delta_1) \geq r(k,\delta_2) &
\end{alignat*}
 \label{eq:nsf}
\end{definition}
We argue that neighbourhoods designed for local search on combinatorial problems typically have the $\mathrm{NSC(k)}$ property, becoming increasingly positively biased for costs towards the optimum.
We illustrate this below with two neighbourhoods, flipping the truth value of a binary variable in satisfiability problems (2-SAT)
and the 2-swap operator in travelling salesmen problems.
This is the property of locality deployed below in the proofs that local search outperforms blind search.

A simple direct consequence of this definition is that, if $\mathrm{NSC}(k)$ and $r(k,k) \geq 1$, then neighbourhood search starting at a point with fitness $k$ is beneficial.
\begin{proof}
Since, by $\mathrm{NSC}(k)$, $\forall \delta < k : r(k,\delta) \geq r(k,k) \ge 1$, and $pn(k,k-\delta) \ge r(k,\delta) \times p(k-\delta)$ therefore $pn(k,k-\delta) \geq p(k-\delta)$.
Thus $\forall i \in 0..k : pn(k,i) \geq p(i)$
Consequently, by definition:
\begin{equation}
    \label{rkk}
r(k,k) \geq 1 \rightarrow pn^<(k) \geq p^<(k)
\end{equation}
\end{proof}


\subsection{Starting cost}
\label{sec:dec_cost}
Neighbourhood search is unlikely to improve any faster than blind search starting from a very poor cost.  

We therefore consider neighbourhood search from a current solution which is already of reasonably good cost. 

We have seen that neighbourhoods with similar cost have neighbourhood operators that only change a small proportion of terms from an objective function that is the sum of many terms. 
Such objective functions occur in the well-known problems listed by Grover \cite{grover1992} above, as well as most (Polynomial time Local Search) PLS-complete problems \cite{michiels2007}, and
NP-hard problems whose cost is the weighted sum of violated constraints.

For a problem instance with such an objective function, the cost probability typically reduces sharply towards the optimum.
If unconstrained, the optimum is reached when all the terms take their minimum value: there is just one such point.
Then there are ${n \choose x}$ ways that $x$ out of $n$ terms take their minimum value, and this number increases by a factor of $\frac{n-x}{x}$ when $x$ decreases by one.
Thus, as the cost increases away from the optimum, the number of combinations of values that reach that sum increases dramatically, thus increasing its probability.
On the other hand, if there are constraints, which exclude a similar proportion of points at each cost level, the same reduction in cost probability occurs towards the optimum.
For VLSI problems, for example, \cite{white1984} showed that the solution costs have a normal distribution over the interval between their minimal and maximal cost, having few solutions with cost near the extremes.

Our proof of the benefit of neighbourhood search, requires that in the current neighbourhood, the cost probability should be decreasing with cost level towards the optimum.
Specifically such problem classes have a moderate cost level, $k_{mod}$, better than which this thinning out occurs.
\begin{definition}
The cost level $k_{mod}$ is the highest cost below than which $p(k)$ is monotonically decreasing\footnote{In this paper we use \emph{monotonically decreasing} to be synonymous with \emph{monotonically nonincreasing}, and similarly \emph{monotonically increasing} means \emph{monotonically nondecreasing}} with decreasing cost: \\
$\forall k_1 \leq k_2 \leq k_{mod} : p(k_1) \leq p(k_2) $
\end{definition}
For many problems 
$k_{mod}$ lies about halfway between $0$ and $k_{max}$.
However, for a problem class whose cost probabilities are uniform ($\forall i,j :  p(i)=p(j)$), the modal cost is the maximum cost, so $k_{mod}=k_{max}$.

Specifically if $k_c$ is the current cost, for beneficial neighbourhood search we require that $p(k)$ should be monotonically increasing with $k$ in the range $0 \ldots 2*k_{c}$.
For such a cost $k_c$, $ p(k_c+\delta): \delta \leq k_c$ monotonically increases with increasing $\delta$, while $p(k_c -\delta)$ decreases.

\begin{definition}
A cost $k$ is \emph{good enough}, and we write GE($k$), if for all $\delta \leq k$, $p(k -\delta)$ decreases with increasing $\delta$, and $p(k+\delta)$ increases with increasing $\delta$
\end{definition}
Thus GE($k$) holds whenever $2 \times k \leq k_{mod}$.

\subsection{Neighbours with no cost difference}
Neighbours' similar cost includes the chance that neighbours have the same cost.
If $pn(k,k)$ is large enough
(i.e. a high enough proportion of the neighbours of a given cost $k$ also have cost $k$), then neighbourhood search may not outperform blind search.

The proofs in the next section include a limit on this proportion sufficient to ensure neighbourhood search is beneficial.

 \section{Probability of Improvement}
 \label{sec:proofs}
\subsection{Definitions and Lemmas}

Let us write $\bar{r}(k)$ for the average value of $r(k,\delta): \delta \in 1..k$.
This is the average NWeight for cost differences up to the optimum:
$$\bar{r}(k) = \sum_{\delta=1}^k r(k,\delta) / k$$

Secondly let us write $pbr^<(k)$ for the NWeighted probability of selecting a neighbour with cost lower than $k$.
$$pbr^<(k) = \sum_{\delta=1}^k p(k-\delta) \times r(k,\delta)$$
This is the minimum probability a neighbour of cost $k$ is improving, assuming the neighbourhood is unbiased.

\vspace{0.3cm}

For the proof the neighbourhood search is beneficial, we start with three lemmas.
The first is that for a good enough current cost $k$, and assuming neighbourhood similar cost, the neighbourhood probability of improving is greater than the probability of improving with blind search times the average NWeight.
\begin{lemma}
\label{lemma:decdec}
Assuming:
\begin{alignat*}{2}
&2 \times k \leq k_{mod} \ \ \ &(GE)\\
&\forall \delta_1 < \delta_2 : r(k,\delta1) \geq r(k,\delta2) \ \ \ &(NSC)\\ 
&\forall \delta : pn(k,k-\delta) \ge r(k,\delta) \times p(k-\delta) \ \ &(NSC)
\end{alignat*}
it follows that
\begin{equation}
    \label{eq:pbrlow}
    pn^<(k) \geq \bar{r}(k) \times p^<(k)
\end{equation}
\end{lemma}
\noindent The result that 
$pbr^<(k) \geq \bar{r}(k) \times p^<(k)$
is proven in the technical appendix.
Since the neighbourhood is unbiased it follows that:
$pn^<(k) \geq pbr^<(k) \geq \bar{r}(k) \times p^<(k)$.

\vspace{0.3cm}
We now introduce the probability of picking a worse neighbour.
Let us define  $pn^>(k)$ and $pbr^>(k)$ for neighbours with worse cost:
$$pn^>(k) = \sum_{\delta=1}^k pn(k, k+\delta) {\rm \ \& \ } pbr^>(k) = \sum_{\delta=1}^k p(k+\delta) \times r(k,\delta)$$
The second and third lemmas reveal that for a good enough current cost $k$, and assuming neighbourhood similar cost, the probability of selecting worse neighbour is less than than the probability of selecting a worse point with blind search, times the average NWeight.
\begin{lemma}
\label{lemma:incdec}
Assuming:
\begin{alignat*}{2}
&2 \times k \leq   k_{mod} \ \ &(GE)\\
&\forall \delta_1 < \delta_2 : r(k,\delta1) \geq r(k,\delta2)\ \  &(NSC)
\end{alignat*}
it follows that
\begin{equation}
    \label{eq:rhi}
    \forall \delta>k : r(k,\delta) < \bar{r}(k)
\end{equation}
\begin{equation}
    \label{eq:pbrhigh}
pbr^>(k) \leq \bar{r}(k) \times p^>(k)    
\end{equation}
\end{lemma}
Result \ref{eq:rhi} follows immediately.
The second result \ref{eq:pbrhigh} is proven in the technical appendix.
Finally, the following lemma follows from our definitions.
\begin{lemma}
\label{lemma:unbiasedk}
If the neighbourhood of $k$ is unbiased, then:
\begin{equation}
\label{eq:unbiashik}
pn^>(k) \leq pbr^>(k)
\end{equation}
\end{lemma}
Proven in the technical appendix.

\subsection{Proofs that neighbourhood search is beneficial}
The monotonicity conditions GE and NSC are only needed to prove the two lemmas \ref{lemma:decdec} and \ref{lemma:incdec}.
 The condition GE, that the cost probability is monotonically decreasing towards the optimum, can be violated by a single high cost probability.
 Similarly the condition NSC, that NWeight is monotonically decreasing with increasing cost-difference, can also be violated by a single high NWeight.
 Thirdly an almost unbiased neighbourhood may be violated at a single distance $\delta$.

To prove that neighbourhood search is beneficial we shall therefore use the conclusions of these lemmas, equations \ref{eq:pbrlow}, \ref{eq:rhi}, \ref{eq:pbrhigh} and \ref{eq:unbiashik}, which hold consistently
even in the above cases which strictly violate GE, NSC and/or unbiased.

\noindent The first beneficial neighbourhood theorem:
\begin{theorem}
\label{thm:avrgt1}
Beneficial neighbourhood search when average NWeight $ \geq 1$\\
If equation \ref{eq:pbrlow}  is satisfied and
$\bar{r}(k) \geq 1$
then
$$pn^<(k) \ge p^<(k)$$
\end{theorem}
\begin{proof} of theorem \ref{thm:avrgt1}\\

$$pn^<(k) \geq \bar{r}(k) \times p^<(k) \ge p^<(k)$$
\end{proof}

In case $\bar{r}(k) \geq 1$, above, there is no limit on the value of $pn(k,k)$.
We next tackle the case $\bar{r}(k) < 1$.
In this case there may be a high proportion of neighbours with the same cost as the current point - in short $pn(k,k)$ may be high.
Assuming the consequence of lemma \ref{lemma:incdec}, 
we can infer a limit on $pn(k,k)$ below which
$pn^<(k) \geq p^<(k)$.
In particular if $pn(k,k) \leq p(k)$, the result follows.

Let us write
\begin{alignat*}{3}
&p^>(k) && {\rm \ \ for \ \ } \sum_{\delta=1}^k p(k+\delta) &\\
&p^{>>}(k)&& {\rm \ \ for \ \ } \sum_{\delta > k} p(k+\delta) &\\
\end{alignat*}

\vspace{-0.5cm}

and define $pn^{>>}(k), pbr^{>>}(k)$ similarly.

\noindent Then
\begin{equation*}
    \begin{split}
1 & = \ \ p^<(k)+p^>(k)+p^{>>}(k)+p(k) \\
& = \ \ pn^<(k)+pn^>(k)+pn^{>>}(k)+pn(k,k)   
    \end{split}
\end{equation*}
so 
\begin{equation}
\label{eq:pngtggt}
\begin{split}
pn^<(k) = \ \ & p^<(k) + (p^>(k)-pn^>(k)) + \\
& (p^{>>}(k)-pn^{>>}(k)) + (p(k)-pn(k,k))
\end{split}
\end{equation}
\begin{definition}
$a^>$ and $a^{>>}$\\
{\rm In the light of the above equation, we define \ } $a^>(k)$ {\rm \ and \ } $a^{>>}(k)$ {\rm \ as follows:}\\
$a^>(k) = (p^>(k)-pn^>(k))$\\
$a^{>>}(k) = (p^{>>}(k) - pn^{>>}(k))$
\end{definition}
\begin{lemma}
\label{lemma:apositive}
If 
equations \ref{eq:rhi}, \ref{eq:pbrhigh} and \ref{eq:unbiashik} all hold, and
$\bar{r}(k) < 1$
then
$$a^>(k)+a^{>>}(k) \geq 0$$
\end{lemma}
\begin{proof} of lemma \ref{lemma:apositive}\\
By equation \ref{eq:unbiashik}, $pn^>(k) \leq pbr^>(k)$ and, by equation \ref{eq:pbrhigh}, since $\bar{r}(k) < 1$, then \ \ $pbr^>(k) \leq p^>(k)$, and it follows that:
$$a^>(k) = (p^>(k)-pn^>(k))  \geq 0$$
Moreover, if $\delta>k$ then $p(k-\delta)=0$ which means \\
$pn(k,k+\delta)= p(k+\delta) \times r(k,\delta)$.\\
By equation \ref{eq:rhi} and by assumption,\\
$\forall \delta>k : r(k,\delta)\leq \bar{r}( k) < 1 $\\
and therefore
$\forall \delta>k : pn(k,k+\delta) < p(k+\delta)$\\
so
$p^{>>}(k) \geq pn^{>>}(k)$.\\
Consequently:
$$a^{>>}(k)  = (p^{>>}(k) - pn^{>>}(k)) \geq 0 $$
\end{proof}

The second beneficial neighbourhood theorem follows.  This theorem shows that for a good enough cost $k$, if $\mathrm{NSC}(k)$ holds, then local search is beneficial unless too many neighbours have the same cost $k$.  Indeed equation \ref{eq:weak_cond} gives a bound on this number.
\begin{theorem}
\label{thm:improvement}
Beneficial neighbourhood search when average NWeight $<1$\\
If equations \ref{eq:rhi}, \ref{eq:pbrhigh}, \ref{eq:unbiashik} all hold and $p(k) \geq pn(k,k)$ then even if $\bar{r}(k) < 1$
$$pn^<(k) \ge p^<(k)$$
\end{theorem}
\begin{proof} of theorem \ref{thm:improvement}\\
From equation \ref{eq:pngtggt} above, we have:
\begin{equation*}
    \begin{split}
pn^<(k) = & \  p^<(k) + (p^>(k)-pn^>(k)) + \\
& \ (p^{>>}(k)-pn^{>>}(k)) + (p(k)-pn(k,k))       
    \end{split}
\end{equation*}

so, by definition:
$$pn^<(k) = p^<(k) + a^>(k) + a^{>>}(k) + (p(k)-pn(k,k))$$
By lemma \ref{lemma:apositive}, $a^>(k)+a^{>>}(k) \ge 0$, and by assumption $p(k) - pn(k,k) \geq 0$
therefore:
$$pn^<(k) \geq p^<(k)$$
\end{proof}
Note that the same proof shows that, even if $\bar{r}(k) < 1$, the neighbourhood search is beneficial under the weaker condition that
\begin{equation}
(pn(k,k)-p(k)) \leq  (a^>(k) + a^{>>}(k)) \label{eq:weak_cond} 
\end{equation}

In general, the value of $a^>(k)+a^{>>}(k)$ is far larger than the value of $pn(k,k)-p(k)$ because they sum the difference between $pn(k,\delta)$ and $p(k+\delta)$ over the whole range of $\delta \in 0 \ldots k_{max}-k$, whereas $pn(k,k)-p(k)$ is simply this difference when $\delta=0$.

\section{Calculating neighbourhood properties for a problem class}
\label{sec:2SAT}

In this section we take a very simple example of a problem class, and show how we can infer its specification and properties.
In particular we show it has the Neighbour Similar Cost property at all costs from $0$ to the modal cost.

The class is a subclass of MAX-2-SAT, where there is a given number of variables and clauses.

Since the same neighbourhood operator --- flipping a boolean --- applies to all instances of this class, we can model the search for an unknown instance of the class.
Assuming the instance is drawn uniformly at random from the class,
the expected values of $p$, $pn$ and $r$ are the same as for the class as a whole.

Specifically we take the class of MAX-2-SAT problems with $50$ variables and $100$ 2-variable clauses, in which each variable appears in exactly $4$ distinct clauses.
Each variable can take the value $true$ or $false$, so there are $2^{50}$ candidate solutions.
The cost of a solution is the number of violated constraints, so the range of cost values is $0 \ldots 100$.
This completes the specification of our example problem class.

Based on the above specification we calculate $p$, $k_{mod}$, $pn$ and $r$.

Each clause is true with probability $3/4$ and false with probability $1/4$.
The probability that $C$ clauses are false is 
$$p(C) = (1/4)^C \times (3/4)^{(100-C)} \times {C \choose 100}$$
The most likely cost is $p(25) = 0.092$, and this is the modal cost $k_{mod}$.
The probabilities are shown in figure \ref{fig:sat2p}
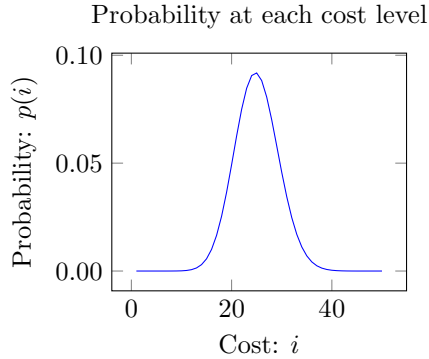
\begin{figure}[htb]
\centering
\begin{tikzpicture}
\begin{axis}[
width=5.5cm,
title= Probability at each cost level,
scaled ticks = false,
xlabel={Cost: $i$},
ylabel={Probability: $p(i)$},
y tick label style={
                /pgf/number format/fixed,
                /pgf/number format/fixed zerofill,
                /pgf/number format/precision=2},
]
\addplot [blue] [solid]  table {sat2p.dat};
\end{axis}
\end{tikzpicture}
\caption{MAX-2-SAT probability at each cost level}
\label{fig:sat2p}
\end{figure}

The neighbours of a solution result from flipping the value of a single variable.  Since a variable only appears in $4$ clauses,
$pn(C,\delta)=0$ for all $C$ and any $\delta>4$.
Flipping a variable in a clause that is false always makes it true - thus increasing the cost by $1$.
Flipping a variable in a true clause makes it false with a probability of $1/3$.
If the current cost is $C$, the probability that the variable to be flipped is in a false clause is $C/100$, and a for true clause it is $(100-C)/100$.  Thus, for example,
$pn(C,C-4) =$\\
$ (C/100)*((C-1)/99)*((C-2)/98)*((C-3)/97)$

For this problem class and neighbourhood operator, we find that neighbourhoods are unbiased for cost values lower than the modal cost $25$.
Writing $posr(C,\delta) = pn(C,\delta) - (p(C-\delta)*r(C,\delta)$,
then, by definition, the neighbourhood of $C$ is unbiased if and only if $posr(C,\delta)$ is zero or positive for all values of $\delta$.
The values for $r(17,\delta)$, and for $posr(17,\delta)$
 are given in table \ref{tab:2satunskew}.
This reveals that the $NSC(17)$ property holds: $r(17,\delta)$ decreases with increasing $\delta$ and that the neighbourhood is positively biased.
\begin{table}[!ht]
    \centering
    \begin{tabular}{|c||c|c|}
    \hline
        $\delta$ & $r(17,\delta)$ & $posr(17, \delta)$ \\
    \hline 
         1 & 12.5 & 0.045  \\
         2 &  5.0 & 0.029  \\
         3 &  1.1 & 0.006  \\
         4 &  0.1 & 0.001  \\
    \hline
    \end{tabular}
    \caption{Probabilities $r(17,\delta), posr(17,\delta)$}
    \label{tab:2satunskew}
\end{table}
Above the modal cost $25$ the neighbourhoods are negatively biased.

By theorem \ref{thm:avrgt1} neighbourhood search is beneficial from a cost $k$ if $\bar{r}(k)$ (the average value of $r(k,\delta): \delta \in 1..K$) is greater than $1$.
Table \ref{tab:avr} gives the values for $\bar{r}(k)$.
\begin{table}[!ht]
    \centering
    \begin{tabular}{|c||c|c|c|c|c|c|}
    \hline
        $k$ & 20 & 17 & 14 & 11 & 8 & 5 \\ 
    \hline 
         $\bar{r}(k)$ & 0.36 & 1.10 & 6.15 & 70.3 & 1950 & 178,000 \\
    \hline
    \end{tabular}
    \caption{$\bar{r}(k)$ - average value of $r(k,\delta)$}
    \label{tab:avr}
\end{table}
Clearly neighbourhood search is beneficial starting at costs of $17$ or better.

For completeness we give the values for $t(k) = pn(k,k)-p(k)$ and for $a(k) = a^>(k)+a^{>>}(k)$ for $k$ up to a cost of $26$ 
in table \ref{tab:2satbeneficial}.
Neighbourhood search is beneficial if $t(k) < a(k)$, so it is beneficial even at costs of $20$ and $23$ where $\bar{r}(k) < 1$:

\begin{table}[!ht]
    \centering
    \begin{tabular}{|c||c|c|c|c|c|c|c|}
    \hline
        $k$ & 26 & 23 & 20 & 17 & 14 & 11 & 8\\ 
    \hline 
         t(k) & 0.19 & 0.20 & 0.23 & 0.27 & 0.28 & 0.27 & 0.26 \\ 
         a(k) & 0.01 & 0.24 & 0.41 & 0.48 & 0.47 & 0.42 & 0.36 \\ 
    \hline
    \end{tabular}
    \caption{Neighbourhood search beneficial if $t(k)<a(k)$}
    \label{tab:2satbeneficial}
\end{table}

\vspace{-0.5cm}
\section{Applying theory to practice}
\label{sec:tsp}
To explore the implications of this theory we generated a travelling salesman instance small enough that we could generate all solutions.
Using 2-opt as the neighbourhood operator, we investigated the NSC properties, and the density of solutions.
For this small example we established that the conditions of theorem \ref{thm:avrgt1} hold, and neighbourhood search is beneficial.

For such real problems, the optimal cost $k_{opt}$ may not be $0$.
Accordingly we adapt the definitions of $p^<(k), pn^<(k), \bar{r}(k)$ in the obvious way.
For example $\bar{r}(k) = \sum_{i=1}^{k-k_{opt}}r(k,i) / (k-k_{opt})$.

\subsection{A 10 city TSP}
We generated a single 10 city TSP, with inter-city edge lengths are randomly generated in the range 1 to $25$, and generated the values for $p(k)$. 
The optimum cost was $k_{opt} = 85$ the modal cost $k_{mod} = 119$.
The furthest cost from the optimum where $GE(k)$ holds was $k_{ge} = max({k : k+(k-k_{opt}) \leq k_{mod}}) = 102$.

The conditions for theorem \ref{thm:avrgt1}, which establish that neighbourhood search at cost $k$ is beneficial 
are:
\begin{enumerate}
\itemsep0em
    \item $\bar{r}(k) \geq 1$ 
    \item $pn^<(k) \geq \bar{r}(k) \times p^<(k)$
\end{enumerate}
We evaluated each of these properties for our TSP 10 problem instance.\\
(1) We show $\bar{r}(k)$, for all costs $k \in k_{opt}+1 \ldots k_{ge}$, the range within which neighbourhood search is expected to be beneficial.\\
(2) We show the values of $pn^<(k)$ and of $\bar{r}(k) \times p^<(k)$ over the same range.

The left hand side of figure \ref{fig:instance10nsf} shows that $\bar{r}(k) \geq 1$ over the whole cost range.
The right hand side shows that also $pn^<(k) \geq \bar{r}(k) \times p^<(k)$ is satisfied over this range.

\begin{figure}[hbt]
\pgfplotsset{
    every axis legend/.append style={
    at={(1,1)},
    anchor=north east,
                                    },
            }
\begin{tikzpicture}
\begin{axis}[
width=5.5cm,
title= Average NWeight,
scaled ticks = false,
xlabel={cost $k$},
ylabel={$\bar{r}(k)$},
ymode=log,log basis y=10,
y tick label style={
                /pgf/number format/fixed,
                /pgf/number format/precision=1},
]
\addplot [green] table {tsp_multi_av_r10251.dat};
\end{axis}
\end{tikzpicture}
\hspace{10pt}
\begin{tikzpicture}
\pgfplotsset{
    every axis legend/.append style={
    at={(0,1)},
    anchor=north west,
                                    },
            }
\begin{axis}[
width=5.5cm,
title= {$pn^<(k) \geq \bar{r}(k) \times p^<(k)$},
scaled ticks = false,
xlabel={instance - value of $k$},
yticklabel style={
        /pgf/number format/fixed,
        /pgf/number format/precision=2
},
legend entries={$pn^<(k)$,$\bar{r}(k) \times p^<(k)$},
]
\addplot [blue] [dashed] table 
{tsp_multi_pn_low10251.dat};
\addplot [green] [solid] table {tsp_multi_pavr_low10251.dat};
\end{axis}
\end{tikzpicture}
\caption{TSP 10 satisfies conditions for theorem \ref{thm:avrgt1}}
\label{fig:instance10nsf}
\end{figure}
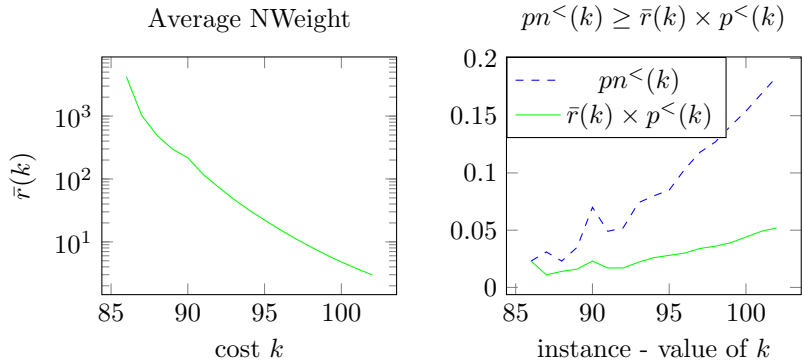

Thus both conditions for beneficial neighbourhood search are satisfied for this TSP10 instance.

\subsection{Other and larger TSPs}
It was proven (equation \ref{rkk} above) that if $r(k,k-k_{opt}) \geq 1$, then neighbourhood search is beneficial at starting cost $k$.

100 TSP10 instances were generated, with edge lengths chosen randomly in the range 
$1 \ldots 25$, and in every TSP10 instance, for every starting cost $k \in k_{opt} \ldots k_{ge}$,
either $r(k,k-k_{opt}) \geq 1$ or properties 1 and 2 were computed.
In every case they proved to be satisfied.

Naturally for larger problems, the decrease in cost probability towards the optimum, and the decrease in NWeight for increasing cost difference is closer to strict monotonicity.  
For illustration we sampled 400000 points and their neighbourhoods in an 80 city TSP.

We generated a single 80 city TSP, with inter-city edge lengths randomly generated in the range 1 to $25$, and inferred the values for $p(k)$ by sampling.
The optimum cost was $k_{opt} = 897$ the modal cost $k_{mod} = 1039$.
The furthest cost from the optimum where we still expect $NSC(k)$ to hold is $max({k : k+(k-k_{opt}) \leq k_{mod}}) = 968$.
We then checked that NWeight $r(968,\delta)$ decreases as $\delta$ increases from 1 to $50$, which is the maximum possible cost change from a 2-swap.

The results are shown in figure \ref{fig:whitetspct80}.
\begin{figure}[!ht]
\begin{tikzpicture}
\begin{axis}[
width=5.5cm,
title= Cost range probability,
scaled ticks = false,
xlabel={cost $k$},
ylabel={p(k)},
y tick label style={
                /pgf/number format/fixed,
                /pgf/number format/fixed zerofill,
                /pgf/number format/precision=3},
]

\addplot [green] table {tsp_approx_p80_400000_1.dat};
\end{axis}
\end{tikzpicture}
\hskip 10pt
\pgfplotsset{
    every axis legend/.append style={
    at={(0,1)},
    anchor=north west,
                                    },
            }
\begin{tikzpicture}
\begin{axis}[
width=5.5cm,
title= {Decreasing $r(987,\delta)$},
scaled ticks = false,
xlabel={value of $\delta$},
ylabel={value of $r(968,\delta)$},
y tick label style={
                /pgf/number format/fixed,
                /pgf/number format/fixed zerofill,
                /pgf/number format/precision=2},
]
\addplot [blue] table {tsp_approx_r_kge80_25_.dat};
\end{axis}
\end{tikzpicture}
\caption{TSP 80 cost probabilities and NWeights}
\label{fig:whitetspct80}
\end{figure}
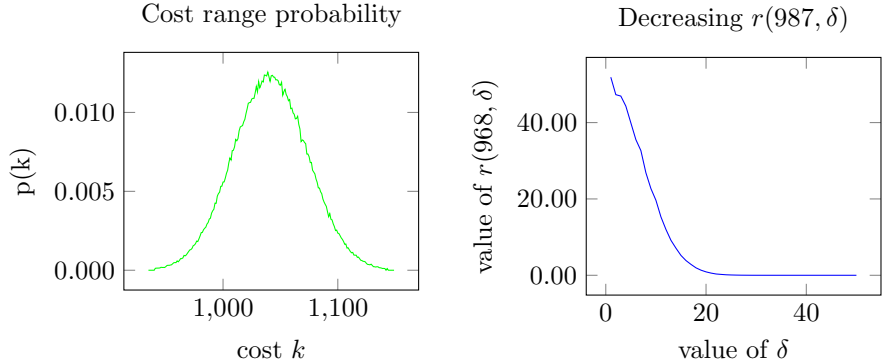

\section{Rate of improvement}
\subsection{Example search space and neighbourhood}
\label{sec:toyex}
Suppose every pair of neighbours has a cost difference within a bound $b$.
This is the L-Lipschitz condition on neighbouring points.
Suppose points with a cost difference less than $b$ all have the same neighbourhood weight:\\ 
$\delta \leq b \rightarrow r(k,\delta) = 1/ \sum_{i =k-b}^{k+b} p(i)$\\
and \\
$\delta > b \rightarrow r(k,\delta) = 0$.

Consider a search space with cost levels $K = 0..200$, where the cost probability is the same at every cost level.
Suppose the current best cost is $30$.
The probability that blind search selects a point with cost better than $30$ is $p^<(30) = 30/201 = 0.149$.
In table \ref{tab:neighsearch} we show the probabilities a neighbour of a point with cost $30$ has better cost $pn^<(30)$:
\begin{table}[!ht]
\scriptsize  
\centering
    \begin{tabular}{| c | c | c c c c c |}
    \toprule
 & $p^<(30)$ & \multicolumn{5}{c |}{$pn^<(30)$}\\
& & $b=1$ & $b=5$ &  $b=10$ & $b=50$ & $b=200$ \\ 
    \midrule
$k=30$ & 0.149 & 0.333 & 0.455 & 0.476 & 0.370 & 0.149\\
    \bottomrule
  \end{tabular}
  \vspace{1ex}
  \caption{ Comparing blind search with neighbourhood search}
  \label{tab:neighsearch}
\end{table}
For all values of $b<200$ neighbourhood search has a greater probability of improving than blind search.

\subsection{Expected Improvement from a single step}
\label{sec:betterblind}

To define the expected improvement from a single step, we consider the probability $pn(k,k')$ of picking a point of cost $k'$.  If $k'$ is lower than $k$, this yields an improvement of $k-k'$.
If, on the other hand, $k'$ is higher than $k$, the neighbour with cost $k'$ is ignored, and there is no ``negative improvement''. 
In short the improvement is $0$. 

The expected improvement from one step $en_{imp}(k)$ is therefore:
\begin{equation}
en_{imp}(k) = \sum_{k'<k} pn(k,k') \times (k-k')
\label{eq:enimp}
\end{equation}
Instead of searching in the neighbourhood, the system could use blind search to try to improve on $k$.  In this case the expected improvement $e_{imp}(k)$ is defined similarly:
\begin{equation}
e_{imp}(k) = \sum_{k'<k} p(k') \times (k-k')
\label{eq:eimp}
\end{equation}

Using the example search space above, we calculate the rate of improvement for $b=1,b=5,b=10,b=50,b=200$, starting at cost $30$.

Table \ref{tab:betterblind} shows the values  for $e_{imp}(30)$ in the first column and $en_{imp}(30)$ for different values of $b$ in the remaining columns:
\begin{table}[!ht]
\scriptsize  
\centering
    \begin{tabular}{| c | c | c c c c c |}
    \toprule
 & $e_{imp}(30)$ & \multicolumn{5}{c |}{$en_{imp}(30)$}\\
& & $b=1$ & $b=5$ &  $b=10$ & $b=50$ & $b=200$ \\ 
    \midrule
$k=30$ & 2.31 & 0.33 & 1.36 & 2.62 & 5.74 & 2.31\\
    \bottomrule
  \end{tabular}
  \vspace{1ex}
  \caption{ Expected one-step improvement from the current optimum  }
  \label{tab:betterblind}
\end{table}
The expected improvement from neighbourhood search only exceeds that from blind search when $b \geq 9$.

To explore an actual combinatorial problem, we investigated a travelling salesman instance small enough that we could generate all solutions.
We generated a single 10 city TSP, with inter-city edge lengths randomly generated in the range 1 to $25$, and we generated the values for $p(k)$. 
The optimum cost was $k_{opt} = 85$ the modal cost $k_{mod} = 119$.

Using 2-opt as the neighbourhood operator, we calculated the values of $e_{imp}(k)$ and $en_{imp}(k)$ for all values of the cost $k \in 85 \ldots 119$.
The results show that, under the assumption that the current cost is the best found so far, local descent has a faster expected rate of improvement than blind search almost up the the modal starting cost.
\begin{figure}[!ht]
\begin{center}
\begin{tikzpicture}
\pgfplotsset{
    every axis legend/.append style={
    at={(0,1)},
    anchor=north west,
                                    },
            }
\begin{axis}[
width=5.5cm,
title= Rate of improvement of blind search,
scaled ticks = false,
xlabel={value of $k$},
y tick label style={
                /pgf/number format/fixed,
                /pgf/number format/fixed zerofill,
                /pgf/number format/precision=2},
legend entries={$e_{imp}(k)$,$en_{imp}(k)$},
]
\addplot [green] table {tsp_eimp_low10251.dat};
\addplot [blue] table {tsp_enimp_low10251.dat};
\end{axis}
\end{tikzpicture}
\hskip 10pt
\pgfplotsset{
    every axis legend/.append style={
    at={(0,1)},
    anchor=north west,
                                    },
            }

\caption{TSP 10 instance, rates of improvement}
\label{fig:tsp10enimp}
\end{center}
\end{figure}
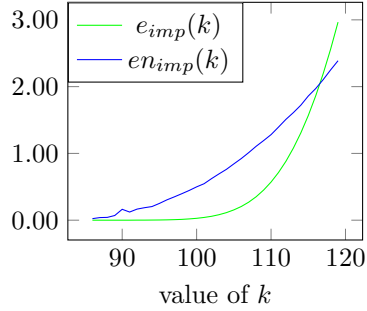

We then generated 100 TSP10 instances, with edge lengths randomly chosen in the range 1..25.
We computed $e_{imp}(k)$ and $en_{imp}(k)$ for values of $k$ from just above the optimum cost\footnote{In some TSP instances, all the points with cost just above the optimum have no improving neighbours, so the rate of improvement is $0$} to the modal cost on each instance.
We recorded the lowest cost $k$ for which $e_{imp}(k) > en_{imp}(k)$. From all 100 of these problem instances the overall lowest such cost was 29 above the optimum, confirming that the results shown in figure \ref{fig:tsp10enimp} are typical.

\section{Local blind search}
\subsection{The drawback of analysing a single step}
When the best cost so far $k$ is better than the cost $i$ of the current point, an improving point is one with cost better than $k$.
In this case the definition of improvement requires two parameters, the current cost and the best cost found so far.  For local descent the rate is
\begin{equation}
\begin{split}
e_{imp}(k,t) &= \sum_{i< t} p(i) \times (t-i)\\
en_{imp}(k,t) &= \sum_{i< t} pn(k,i) \times (t-i)
\end{split}
\label{eq:enimp2}
\end{equation}

On the toy example of section \ref{sec:toyex}, if the current cost is $30$ and the best cost so far is $15$ then the comparison between the rate of improvement of blind search and local descent is shown in table \ref{tab:impnotbest}.
As before we consider neighbourhoods where the biggest difference between neighbours $b$ is $1, 5, 10, 50$ or $100$.

\begin{table}[!ht]
\scriptsize  
\centering
    \begin{tabular}{| c | c | c c c c c |}
    \toprule
 & $e_{imp}(15)$ & \multicolumn{5}{c |}{$en_{imp}(30,15)$}\\
& & $b=1$ & $b=5$ &  $b=10$ & $b=50$ & $b=200$ \\ 
    \midrule
$k=30$ & 0.6 & 0.0 & 0.0 & 0.0 & 1.19 & 0.6\\
    \bottomrule
  \end{tabular}
  \vspace{1ex}
  \caption{ Expected one-step improvement from a point worse than the current optimum  }
  \label{tab:impnotbest}
\end{table}
Table \ref{tab:impnotbest} shows that improvement in a single step does not reflect the benefit of local descent over a sequence of steps.

Instead we investigate the number of steps required to reach a given cost, and explore the property of neighbourhoods sufficient to make local descent beneficial. 
We say local descent is beneficial if the number of steps to reach a given cost is lower with local descent than with blind search.

\subsection{Definitions and properties}
When analysing the expected number of steps for a local descent, we assume a fixed neighbourhood size $n$.
We assume the first improving neighbour found by selecting uniformly at random from the neighbourhood of the current point becomes the new current point.
If at any point there is no improving neighbour, the local descent ends and we assume the algorithm falls back on blind search.
To be unambiguous we term this \emph{local blind descent}.

Viewing local search as a combination of exploration and exploitation 
\cite{explore-exploit},
we note that local blind descent deploys exploration using blind search, then exploitation using local descent, and only returns to exploration as a final step when it reverts to blind search. 

\subsubsection{\emph{blind(t)}: The expected number of steps for blind search}
Blind search arbitrarily selects a point in the search space, returns its cost, and then tries again.
In this case the expected number of steps $blind$ for blind search to find a point with cost $t$ or better is:\\
\begin{equation*}
\begin{split}
     blind(t) &= \sum_{i=1}^{\infty} i \times ((1-p^<(t+1))^{i-1} \times p^<(t+1))\\
     &= 1/p^<(t+1)  
\end{split}
\end{equation*}

\subsubsection{\emph{imp(k,n)}: The expected number of steps for neighbourhood search to improve}

If the current point has cost $k$, we determine $imp(k,n)$ which is the expected number of steps to improve with a neighbourhood size of $n$.
$$ imp(k,n) = \sum_{j =1}^{n} : j \times (1 - pn^<(k))^{j-1} \times pn^<(k) $$
$imp(k,n)$ is the expected number of steps to improve, \emph{assuming there is an improving neighbour}.
However the probability $nop(k,n)$ that there is no improving neighbour is
$$nop(k,n) = (1 - pn^<(k))^n$$
Thus the probability there is an improving neighbour is $1-nop(k,n)$.

\subsubsection{Formalising local blind descent}
\label{sec:formalising}

We must take into account the possibility that local descent fails to reach the target cost.
Accordingly we analyse an extended local blind descent which behaves as follows.
First blind search is deployed until a point with a good enough cost to start local descent is reached.
For this we fix a ``starting cost'' $k$ and a point found by blind search is good enough if its cost is $k$ or better.
The expected number of blind search steps for this is $blind(k)$.
The point reached by this blind search has cost $j \leq k$ with probability $p(j)/p^<(k+1)$.
\begin{definition}[Local blind descent] 
The expected number of steps for extended local blind descent, with target cost $t$, neighbourhood size $n$ and starting local descent cost $k$, is $lbd(k,t,n)$.
\begin{equation*}
\begin{split}
        lbd(k,t,n) =& blind(k) + \\ 
        & \sum_{j=0}^k steps(j,k,t,n) \times p(j)/p^<(k+1)
        \end{split}
\end{equation*}
\label{def:lbd}
\end{definition}

$steps(j,k,t,n)$ models a search which starts at a point with cost $j$, and chooses neighbours of points with that cost until either none of the $n$ neighbours are improving, with probability $nop(j,n)$, in which case it resorts to blind search, with expected number of steps $lbd(k,t,n)$ or there is an improving neighbour found, with probability $1-nop(j,n)$ after $imp(j,n)$ steps.
In the latter case, 
weighted by the probability that the next point has cost $i$, $steps(i,k,t,n)$ calculates the remaining steps.
\begin{definition}[Steps]
The function $steps(j,k,t,n)$ encodes the expected number of steps, starting with at a point with cost $j$ to reach a point with a cost $t$ or better, assuming all points have a neighbourhood of size $n$,  by local blind descent.
\begin{equation*}
    \begin{split}
&steps(j,k,t,n) =\\
&\left.
  \begin{cases}
  0 & \text{if} \  j \leq t \\
  nop(j,n) \times (n + lbd(k,t,n)) \ \ + &\\
  (1-nop(j,n)) \times & \text{if} \ j > t\\
  \ \ (imp(j,n) + \sum_{0<i<j} \frac{pn(j,i)}{ pn^<(j) } \times steps(i,k,t,n)) 
    \end{cases}
    \right\}
\end{split}
\end{equation*}
\label{def:stepsdef}
\end{definition}

\begin{definition}[Beneficial local blind descent]
We say local blind descent with neighbourhood size $n$, 
starting local descent at cost $k$ or better
is beneficial if 
$lbd(k,t,n) \leq blind(t)$
\end{definition}

\subsection{Full NSC for local blind descent}
\label{sec:nsf2}
For beneficial local blind descent it is also necessary for $r(k,\delta)$ to increase as $k$ decreases.  We show this with a counterexample.

Suppose a local blind descent starts at a cost level $k$ where $r(k,k-t) < 1$,
so the probability a neighbour of $k$ has the target cost is lower than the probability blind search selects a point with the target cost.
If, also, $pn^<(k) \geq p^<(k)$, then from $NSC(k)$ we can infer that $r(k,1) > 1$
Now supposing $r(k,\delta)$ does not increase as $k$ decreases,
and $\forall t \leq k_1<k, \forall \delta : r(k_1,\delta)=1$.
Then after 
finding an improving neighbour,
local blind descent will have the same expected number of steps as blind search.
In this case, starting at cost level $k$, even if local blind descent reaches the target cost $t$ without restarting, blind search has a smaller expected number of steps.
Technically, in this case, local blind descent is not beneficial.

We therefore require neighbourhood weight, 
not to decrease with decreasing cost $k_1$:
$$\forall k \geq k_1 > k_2 \geq \delta : r(k_1,\delta) \leq r(k_2,\delta)$$

\begin{definition}
Full $NSC(k)$ holds if for all $k_1 \leq k$, the neighbourhood of $k_1$ is unbiased or positively biased, and the neighbourhood weight $r(k,\delta)$ increases as either $k$ or $\delta$ decreases: 
\begin{alignat*}{2}
&\forall \delta \leq k_1 \leq k : pn(k_1,k_1-\delta) \ge r(k_1,\delta) \times p(k_1-\delta) &\\
&\forall k \geq k_1 \geq \delta_1 \geq \delta_2  : r(k_1,\delta_1) \leq r(k_1,\delta_2) &\\
&\forall k \geq k_1 \geq  k_2 \geq \delta  : r(k_1,\delta) \leq r(k_2,\delta) &
\end{alignat*}
 \label{eq:fullnsf}
\end{definition}

We illustrate that full NSC typically holds from $k_{ge}$ 
with the randomly generated TSP.

\begin{figure}[hbt]
\begin{center}
\pgfplotsset{
    every axis legend/.append style={
    at={(1,1)},
    anchor=north east,
                                    },
            }

\begin{tikzpicture}
\begin{axis}[
width=5.5cm,
title= {TSP 100 - $r(k,\delta)$},
xlabel={Cost change: $\delta$},
ylabel={neighbourhood weight: $r(k,\delta)$},
legend entries={{k=2258},{k=2248},{k=2238},{k=2228},{k=2218}},
]
\addplot [blue, solid, line width = 2.5pt] table {mytspr10010225.dat};
\addplot [blue, solid, line width = 2.0pt]  table {mytspr10010224.dat};
\addplot [blue, solid, line width = 1.5pt ] table {mytspr10010223.dat};
\addplot [blue, solid, line width = 1.0pt ] table {mytspr10010222.dat};
\addplot [blue, solid, line width = 0.5pt ] table {mytspr10010221.dat};
\end{axis}
\end{tikzpicture}
\caption{neighbourhood weights}
\label{fig:rki}
\end{center}
\end{figure}
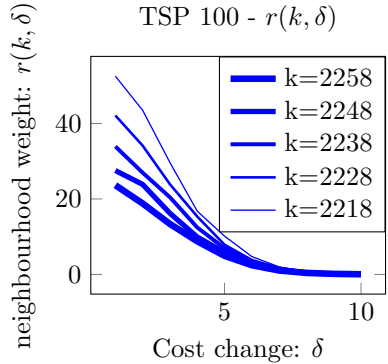

The graph in
figure \ref{fig:rki} shows values of $r(k,\delta)$ in a 100-location TSP. 
The x-axis shows increasing values of $\delta$, and different values of $k$ are shown as different lines.
Again higher values of $k$ yield lower values of $r(k,\delta)$ for all values of $\delta$.
In this graph there is a different curve for each value of $k$, showing how larger values of $k$ yield smaller values for $r(k,\delta)$.
The neighbourhood is 2-swap, and the data is based on all the neighbours of 20 points.
The horizontal axis shows how $r(k,\delta)$ also decreases with increasing $\delta$ (confirming the NSC property).
The graph shows $r(k,\delta) : \delta \in 1..10$, for 5 different values of $k$: 
$2258, 2248,2238,2228,2218$.

\subsection{Theorems on the benefit of local blind descent}
\label{sec:abstract-local-descent}
 
\subsubsection{Conditions guaranteeing that local blind descent is beneficial}
If local descent only starts at the target cost $t$, then local blind descent is the same as blind search.  Thus $lbd(t,t,n) = blind(t)$ and is, by definition, beneficial.

A more general condition  under which local blind descent is beneficial, 
requires
local descent to have a greater chance of selecting a
point with cost t or less than blind search:
\\$\sum_{i=0}^{t} r(k,k-i) \times p(i) \geq p^<(t+1)$.\\
In this case, by the definition of \emph{Full NSC(k)}, the same holds for all cost levels $k_1$ better than $k$: 
\begin{equation*}
    \begin{split}
\sum_{i=0}^t pn(k_1,i) &\geq \sum_{i=0}^t r(k_1,k_1-i) \times p(i) \\
&\geq \sum_{i=0}^t r(k,k-i) \times p(i) \\
&\geq p^<(t+1)    
    \end{split}
\end{equation*}
Consequently at every search step during local descent there is an equal or better chance of reaching the target cost than with blind search.
Naturally during the blind search steps, the same is true, so 
local blind descent is also beneficial under the above condition.

\subsubsection{Large neighbourhoods}

Assuming $k \leq k_{mod}$ and $r(k,1)>0$, then $pn^<(k)>0$.
Suppose our neighbourhoods have infinite size, and $pn^<(k)>0$,
then the probability $nop(k,\infty)$ that there are no improving neighbours is $0$.
In this case local blind descent is guaranteed to reach \emph{any} given cost level right down to the optimum.

More generally, if the neighbourhood size $N$ is large enough, until $pn^<(k)$ is very small, the probability there are no improving neighbours becomes small enough to be ignored in calculating the expected number of steps.
In this section we shall use $\infty$ to represent the size of any large enough neighbourhood.

Accordingly the expected number of steps to improve is:
$$ \sum_{j =0}^{\infty} : (1+j) \times (1 - pn^<(k))^j \times pn^<(k) = 1/pn^<(k)$$
Simplifying, yields the following equation for $imp(k) = imp(k,\infty)$, the expected number of steps to improve from cost $k$:
\begin{equation}
    imp(k,\infty) = imp(k) = 1 / pn^<(k)
    \label{imp-repl}
\end{equation}

Since every infinite neighbourhood of a non-optimal point 
has an improving neighbour 
we can simplify the specification of the expected number of local descent steps.
Since the local descent cannot fail, so blind search is never restarted, the number of steps required by local descent is independent of the starting cost.

Writing $steps(k,t) = steps(k,s,t,\infty)$, for any starting cost $s$, we have:
\begin{equation*}
    \begin{split}
&steps(k,t) = \\
&\left.
  \begin{cases}
  0 & \text{if} \  k < t \\
  imp(k) \times (1 + \sum_{i=0}^{k-1} : pn(k,i) \times steps(i,t) & \text{if} \ k \geq t
    \end{cases}
    \right\}
\end{split}
\end{equation*}
(since $imp(k) = 1/pn^<(k)$)

\begin{definition}
We define the average neighbourhood weight of $k$ down to cost $t$ as
$$\bar{r}(k,t) = (\sum_{i=1}^{k-t} : r(k,i)) / (k-t) $$
\end{definition}
Given a large enough neighbourhood size, local blind descent starting at a cost $k$ and improving to a cost $t$ or better, has a smaller expected number of steps than blind search under
the following conditions:
\begin{theorem}
Assuming:
\label{thm:localdescent}
\begin{equation*}
    \begin{split}
&\bar{r}(k,t) \geq p^<(t+1) / p(t)\\
&k \leq k_{mod}\\
&\mathit{Full \ \ NSC(k)}
    \end{split}
\end{equation*}
it follows that
\begin{flalign*}
\ \ \ steps(k,s,t,\infty) = steps(k,t) & \leq blind(t)&
\end{flalign*}
\end{theorem}
The proof is in the technical appendix.

\subsubsection{Benchmark example}
Suppose, on the other hand, the neighbourhoods are not so large.
To compute the expected number of steps required by local blind descent, we need an estimate of the cost probabilities $p(k):k \in K$ the neighbourhood size $n$ and the neighbourhood weights $r(k,\delta): k \in K, \delta \leq k$.

To illustrate the performance of local blind descent, we introduce a benchmark example, intended to be representative of combinatorial problems and Full $NSC$ neighbourhoods.

We discuss concrete results on this benchmark firstly because it can give insights into local search behaviour on typical combinatorial problems.  Secondly we show that from good estimates of the cost probability distribution of a problem, and its neighbourhood size and weights, we can have some insights into restarting local search.

\paragraph{Benchmark problem class - cost probabilities}
For a problem instance whose objective is the sum of many terms, the optimum is reached when all the terms take their minimum value: there is just one such point.
If each term could take values $0$ or $1$,
then there are ${M \choose x}$ ways that $x$ out of $M$ terms take the value $1$, giving an objective value of $x$.

Using this as a representative of a combinatorial problem, we set $ct(4k) = {50 \choose k}$
and explore the benefit of local blind descent on this problem class\footnote{If $i \in 1 \ldots 3$ then $ct(4k + i)$ lies (linearly) between ${50 \choose k}$ and ${50 \choose k} + 1$\\
Setting $ct(k)={200 \choose k}$ exceeds the precision of the computer, \\
because in this case $p(0)=6e^{-61}$ and \ $1-p(0)$ returns \ $1.0$}.
Thus we have a cost range of $k \in 0 \ldots 200$,
with an exponentially decreasing cost probability towards the optimum.

\paragraph{Benchmark - neighbourhood probabilities}
Our benchmark neighbourhoods satisfy the L-Lipschitz condition as in the example of section \ref{sec:toyex} above, where the maximum cost difference between neighbours is $b$.
If $b=3$, the neighbourhood of a cost of $30$ for example includes no neighbours better than $27$.  Consequently, given a poor starting cost and a small L-Lipschitz bound $b$, blind search is likely to find better solutions than neighbourhood search.

If the Lipschitz bound is set to $200$, and neighbourhood count is large enough (we used $50$), then the expected number of steps to reach the target cost of $1$ is the same (up to an error of $10^{-14}$) from all starting costs.

\subsubsection{Starting cost from which local descent is beneficial}
\label{sec:descexp}
 Setting the target cost $t=10$ and neighbour count $n=50$, we can calculate the  expected number of steps to reach the target cost, given the starting cost $k$ and the Lipschitz bound $b$.

For each odd value of $b \in 1 \ldots 17$ we calculated the fraction of blind search steps saved by using local blind descent with the best starting cost.  The value $1.0$ means that the number of steps used by local blind descent is \emph{very} small compared to blind search.

The results, for target cost $t=10$, are shown in table \ref{tab:ebdsearch}
 \begin{table}[!ht]
\scriptsize  
\centering
    \begin{tabular}{| c | c c c c c c c c c|}
    \toprule
 Lipschitz  & 1 & 3 & 5 & 7 & 9 & 11 & 13 & 15 & 17 \\ 
 bound\\
     \midrule
 Starting & 23 & 37 & 28 & 23 & 22 & 23 & 24 & 26 & 28 \\
 cost \\
    \midrule
 Savings & 1.00 & 1.00 & 1.00 & 0.96 & 0.79 & 0.55 & 0.35 & 0.23 & 0.15 \\
    \bottomrule
  \end{tabular}
  \vspace{1ex}
  \caption{ Maximum savings from local blind descent}
  \label{tab:ebdsearch}
\end{table}

\subsubsection{Restarting cost from which local descent is beneficial}
As a local search progresses, the target cost $t$ is reduced.  For a given value of $b=7$, the following table shows how the optimum starting cost $k$  and the savings changes as $t$ is reduced.

 \begin{table}[!ht]
\scriptsize  
\centering
    \begin{tabular}{| c | c c c c c c c c c|}
    \toprule
 Target  & 0 & 5 & 10 & 15 & 20 & 25 & 30 & 35 & 40 \\ 
 cost\\
    \midrule
 Starting &  8 & 16 & 23 & 32 & 42 & 55 & 66 & 69 & 71 \\
 cost\\
     \midrule
 Savings & 0.68 & 0.87 & 0.96 & 0.99 & 1.00 & 1.00 & 1.00 & 1.00 & 1.00 \\ 
    \bottomrule
  \end{tabular}
  \vspace{1ex}
  \caption{ Maximum savings as the target cost reduces}
  \label{tab:tebdsearch}
\end{table}

The results, for Lipschitz bound $b=7$ in table \ref{tab:tebdsearch} show that when starting with a target cost further from the optimum, local search is most beneficial from an earlier starting cost.  However as the target approaches the optimum, local search needs a much better starting cost to be maximally beneficial.

These results give evidence supporting local descent restarting methods that restart from a reasonably good solution found previously.

\section{Conclusion}
\label{sec:conclusion}
The ``neighbourhood similar cost'' ($\mathrm{NSC(k)}$) property, made precise in definition \ref{sec:nsf}, is not only intuitive, but also sufficient to support the proofs that neighbourhood search has a better chance of improving from the current level of cost than blind search.
$\mathrm{NSC(k)}$ is a property of the cost level $k$.
The concept of neighbourhood weight $r(k,\delta)$ is introduced, which is the increased probability neighbours of a point with cost $k$ have similar cost (differing by $\delta$).
$\mathrm{NSC(k)}$ only holds if $r(k,\delta)$ decreases with increasing $\delta$.

Neighbourhood search is proven to be beneficial, assuming NSC($k$) holds, at any cost $k$ around which the decrease in cost probability $p(i)$ towards the optimum is monotonic. 

Moreover weaker conditions, implied by $NSC(k)$, are sufficient to imply these results.
The decreases in probability of $r(k,\delta)$ and of $p(i)$ need not be strictly monotonic for the proofs to go through.

A class of MAX-2-SAT problems is presented in which the NSC properties are shown to hold.
$100$ small Travelling Salesmen Problems have been randomly generated to show that the conditions sufficient to support beneficial local search hold in all of them.
Sampling solutions and neighbourhoods from larger TSPs suggests they also meet the conditions for our proofs that neighbourhood search is beneficial.

When the current cost is the best one yet found, we introduce a property of neighbourhoods sufficient to guarantee that local descent is beneficial, in the sense that its expected improvement is greater than that of blind search.

When the best cost yet found, $t$, is better than the current cost, $k$, then rather than the rate of improvement, a better measure of progress is the expected number of steps to find a point with cost better than $t$.
We formalise "local blind descent" which reverts to blind search if it reaches a locally optimal point worse than the target cost $t$.

A property of neighbourhood weight is presented which ensures local blind descent is beneficial under this measure of progress.
A more general property of neighbourhood weight is also shown to ensure local search is beneficial if the neighbourhood size is large enough (the proof, given in the appendix, uses infinite neighbourhoods).

Finally, the paper employs a benchmark problem class to investigate local descent with blind restarts. Blind search is used until a ``starting'' cost is reached, at which point the search uses local descent to reach a target cost.
The nearer the target cost is to the optimum, the lower the starting cost that must be reached before local descent is beneficial.
This reveals the drawback of blind restarts, and gives some supporting evidence for restarting from a good solution found previously.

In future research we will explore restarting search at previous solutions, and the probability, in our model, of avoiding local optima that have been reached before.

\bibliographystyle{apalike}
\bibliography{aij}
\begin{appendix}
\section{Proof of Lemmas}
\label{sec:improve-proof}

\noindent {\bf Lemma 1}.\\ 
Recall
$pbr^<(k) = \sum_{\delta=1}^k p(k-\delta) \times r(k,\delta)$ and $p^<(k) = \sum_{\delta=1}^k p(k-\delta)$\\
Assuming:
\begin{alignat*}{3}
&2 \times k &&\leq  \ k_{mod} &(GE)\\
&\forall \delta_1 < \delta_2 : r(k,\delta1) &&\geq \  r(k,\delta2) &\ \ \ \ (NSC) 
\end{alignat*}
it follows that
$$pbr^<(k) \geq \bar{r}(k) \times p^<(k)$$

\begin{proof}
For brevity, we write $p_\delta$ for $p(k-\delta)$, $r_\delta$ for $r(k,\delta)$ and $\bar{r}$ for $\bar{r}(k)$.\\
Since $r_\delta$ is decreasing with increasing $\delta$, we set $x$ to be the largest index for which $r_x \geq \bar{r}$.  Thus $\delta \leq x \rightarrow r_{\delta} \geq \bar{r}$.\\
Since $k-\delta \leq k$ and $2 \times k \leq k_{mod}$, 
$p_\delta$ is also decreasing with increasing $\delta$ so\\
$\delta \leq x \rightarrow (p_\delta - p_x) \times (r_\delta - \bar{r}) \geq 0$ because $p_\delta-p_x$ and $r_\delta - \bar{r}$ are both positive\\
$\delta > x \rightarrow (p_\delta - p_x) \times (r_\delta - \bar{r}) \geq 0$ because $p_\delta-p_x$ and $r_\delta - \bar{r}$ are both negative\\
\vspace{0.5ex}\\
Therefore:
\vspace{-0.5ex}
\begin{equation*}
\begin{split}
   pbr^<(k) = & \sum_{\delta=1}^k p_\delta \times r_\delta \\
  = & \sum_{\delta=1}^x (p_\delta - p_x) \times (r_\delta - \bar{r}) + \sum_{\delta=x+1}^k (p_\delta - p_x) \times (r_\delta - \bar{r})\\
    & + \sum_{\delta=1}^k p_\delta \times \bar{r} + \sum_{\delta=1}^k p_x \times r_\delta - \sum_{\delta=1}^k p_x \times \bar{r} \\
  \geq & \sum_{\delta=1}^k p_\delta \times \bar{r} + p_x \times ( \sum_{\delta=1}^k r_{\delta} - k \times \bar{r}) \\
  = & \sum_{\delta=1}^k p_\delta \times \bar{r} = \bar{r}(k) \times p^<(k)
    \end{split}
\end{equation*}
\end{proof}

\vspace{-1.5ex}

\noindent {\bf Lemma 2}.\\ 
Assuming
\begin{alignat*}{3}
&2 \times k &&\leq  k_{mod} &(GE)\\
&\forall \delta_1 < \delta_2 : r(k,\delta1) &&\geq r(k,\delta2) &\ \ \ \ (NSC) 
\end{alignat*}
it follows that
$$ \sum_{\delta=1}^k p(k+\delta) \times r(k,\delta)  \leq \bar{r}(k) \times \sum_{\delta=1}^k p(k+\delta) $$

\begin{proof}
For brevity, we write $p_\delta$ for $p(k+\delta)$, $r_\delta$ for $r(k,\delta)$ and $\bar{r}$ for $\bar{r}(k)$.\\
Since $\forall \delta : k+\delta \leq 2 \times k \leq k_{mod}$, then $p_\delta$ is increasing with increasing $\delta$ so\\
$\delta \leq x \rightarrow (p_\delta - p_x) \times (r_\delta - \bar{r}) \leq 0$ because only $p_\delta-p_x$ is negative\\
$\delta > x \rightarrow (p_\delta - p_x) \times (r_\delta - \bar{r}) \leq 0$ because only $r_\delta - \bar{r}$ is negative\\
\vspace{0.5ex}\\
Therefore:
\begin{equation*}
\begin{split}
    & \sum_{\delta=1}^k p_\delta \times r_\delta \\
  = & \sum_{\delta=1}^x (p_\delta - p_x) \times (r_\delta - \bar{r}) + \sum_{\delta=x+1}^k (p_\delta - p_x) \times (r_\delta - \bar{r})\\
    & + \sum_{\delta=1}^k p_\delta \times \bar{r} + \sum_{\delta=1}^k p_x \times r_\delta - \sum_{\delta=1}^k p_x \times \bar{r} \\
  \leq & \sum_{\delta=1}^k p_\delta \times \bar{r} + p_x \times ( \sum_{\delta=1}^k r_{\delta} - k \times \bar{r}) \\
  = & \sum_{\delta=1}^k p_\delta \times \bar{r}
    \end{split}
\end{equation*}

\end{proof}

\pagebreak

{\bf Lemma 3}.\\ 
Recall $pn^>(k) = \sum_{\delta=1}^k pn(k, k+\delta) {\rm \ \& \ } pbr^>(k) = \sum_{\delta=1}^k p(k+\delta) \times r(k,\delta)$\\
If the neighbourhood of $k$ is unbiased, then:
$$pn^>(k) \leq pbr^>(k)$$

\begin{proof}
Since $r(k,\delta) = (pn(k,k+\delta)+pn(k,k-\delta)) / (p(k+\delta)+p(k-\delta))$\\
then\\
$pn(k,k+\delta)+pn(k,k-\delta) = r(k,\delta) \times p(k+\delta) + r(k,\delta) \times p(k-\delta)$.\\
Therefore \\
$pn(k,k+\delta) - r(k,\delta) \times p(k+\delta) = - (pn(k,k-\delta) - r(k,\delta) \times p(k-\delta))$.\\
Consequently\\
$pn^>(k) - pbr^>(k) = -( pn^<(k) -pbr^<(k)$\\
Since the neighbourhood is unbiased:\\ 
$pn^<(k) \geq pbr^<(k)$\\
and we conclude that 
$pn^>(k) \leq pbr^>(k)$.
\end{proof}

\section{Proof of the benefit of local blind descent with infinite neighbourhoods}
\label{sec:localdescentproof}
\noindent Recall the definition of the average neighbourhood weight of $k$ down to cost $t$:\\
$$\bar{r}(k,t) = (\sum_{i=1}^{k-t} : r(k,i)) / (k-t) $$
For any given cost threshold $t$, starting cost $k \in t \ldots k_{mod}$, where $\mathit{Full \ NSC(k)}$ and $\bar{r}(k,t) \geq p^<(t+1)/p(t)$,
we prove that
$$steps(k,t) \leq blind(t)$$ 

The proof is structured using a sequence of lemmas.
Firstly we introduce the syntax $$steps_{r,p}(k,t)$$ for $steps(k,t)$ in a problem class where
$r$ denotes the neighbourhood weights $r(k',\delta)$ for $k' \in t+1 \ldots k$, $\delta \leq k'$, and $p$ denotes the cost probabilities $p(i) : i \in 0 \ldots k$.
Thus $steps(k,t) = steps_{r,p}(k,t)$ for some $r,p$.

The syntax
$$steps^u_{r,p}(k,t)$$ denotes the variation where the neighbourhood is unbiased ($\forall k_2 < k_1 \leq k : pn(k_1,k_2) =p(k_2) \times r(k_1,(k_1-k_2))$) and the cost probabilities above the target cost are uniform.
If $p$ denotes $p(i) : i \in 0 \ldots k$, we write
\[
p^u(i) = \left.
  \begin{cases}
  p(i) & \text{if} \  i \leq t \\
  p(t) & \text{if} \ k \geq i > t
    \end{cases}
    \right\}
\]

We write $imp^u_{r,p}(k)$ for the probability of improving from level $k$ under these assumptions
\begin{definition}
Probability of improving with costs $p^u$:\\ 
    $imp^u_{r,p}(k) = 1 / \sum_{i=1}^{k} r(k,k-i) \times p^u(i) $
\end{definition}
Thus $steps^u_{r,p}(k,t)$ 
\[
=  
\left.
  \begin{cases}
  0 & \text{if} \  k \leq t \\
  imp^u_{r,p}(k) \times (1 + \sum_{i=0}^{k-1} : p^u(i) \times r(k,(k-i)) \times steps^u_{r,p}(i,t) & \text{if} \ k > t
    \end{cases}
    \right\}
\]

Lemma \ref{lemma:main} {\bf Fixed count steps} is:\\
Assuming $\mathit{ Full \ NSC(k)}$
    $$steps^u{r,p}(k,t) \leq imp^u_{r,p}(k) \times (k-t)$$

The second lemma establishes that the original count of steps for neighbourhoods weights $r$ and cost probability $p(t)$ (assuming $\mathit{Full \ NSC}$) is smaller than $steps^u_{r,p}$.

Lemma \ref{lemma:reducedsteps} {\bf Reduced steps} is:\\
Assuming $\mathit{ Full \ NSC(k)}$ and $k \in t \dots k_{mod}$
    $$\forall i \in t \ldots k :steps_{r,p}(i,t) \leq steps^u_{r,p}(i,t)$$    

The third lemma establishes that if $\bar{r}(k,t) \geq p^<(t+1)/p(t)$ then the expected number of steps using blind search is greater than $imp^u_{r,p}(k) \times (k-t)$

Lemma \ref{lemma:gandt} {\bf Steps upper bound} is:\\
Assuming $\mathit{ Full \ NSC(k)}$, $k \in t \ldots k_{mod}$ and $\bar{r}(k,t) \geq p^<(t+1)/p(t)$
$$imp^u_{r,p}(k) \times (k-t) \leq blind(t)$$

The final theorem {\bf Beneficial local blind descent} follows from lemma \ref{lemma:main}, lemma \ref{lemma:reducedsteps}  and lemma \ref{lemma:gandt}.

Assuming:\\
$\bar{r}(k,t) \geq p^<(t+1) / p(t)$\\
$k \in t \ldots k_{mod}$\\
$\mathit{Full \ \ NSC(k)}$\\
it follows that
\begin{flalign*}
\ \ \ steps(k,t) = steps_{r,p}(k,t) & \leq blind(t)&
\end{flalign*}

\subsection{Proof of lemma \ref{lemma:main} {\bf Fixed count steps}}
\begin{lemma}[Fixed Count Steps]
Assuming 
\begin{flalign*}
\forall \delta_1 &\leq k_1, \delta_1 \leq \delta_2, \delta_2 \leq k_2, k_2 \leq k :  r(k_1,\delta_1) \geq r(k_2,\delta_2) \ \ \rm{(FullNSC(k))}&
\end{flalign*}
it follows that

$$ steps^u_{r,p}(k,t) \leq imp^u_{r,p}(k) \times (k-t) $$
\label{lemma:main}
\end{lemma}

\begin{proof} 
The proof is by induction, using the definition of $steps_u(k,t)$ to make the inductive step.
\[
steps^u_{r,p}(k,t) = \left.
  \begin{cases}
  0 & \text{if} \  k \leq t \\
  imp^u_{r,p}(k) \times (1 + \sum_{i=0}^{k-1} : p^u(i) \times r(k,(k-i)) \times steps^u_{r,p}(i,t) & \text{if} \ k > t
    \end{cases}
    \right\}
\]
Firstly, we establish the base case $k=t+1$.
In this case lemma (\ref{lemma:main}) holds because 
\begin{equation*}
    \begin{split}
steps_u(t+1,t) &= imp^u_{r,p}(t+1) \times (1+0) = ((t+1)-t) \times imp^u_{r,p}(t+1) \\ 
    \end{split}
\end{equation*}

Proof of inductive step.
Firstly note that
\begin{equation}
    \begin{split}
(i-t) \times imp^u_{r,p}(i) &= (i-t) / (\sum_{j=1}^{i} r(i,i-j) \times p^u(j)) \\
& \leq (i-t) / (\sum_{j=1}^{i-t} r(i,i-j) \times p^u(j)) \\
&= (i-t) / ((i-t) \times \bar{r}(i,t) \times p(0)) \\
&= 1/(\bar{r}(i,t) \times p(0))
    \end{split}
    \label{eq:impbarr}
\end{equation}
Note also that since $r(i,j)$ is decreasing with $i$ by FullNSC(k):
\begin{equation*}
\begin{split}
    \bar{r}(k,t+1) = & \sum_{j=1}^{k-(t+1)} r(k,j)/(k-(t+1))\\
    \leq & \sum_{j=1}^{k-(t+1)} r(k-1,j)/(k-(t+1))\\
    = & \sum_{j=1}^{(k-1) - t} r(k-1,j)/((k-1)-t)\\
    = & \bar{r}(k-1,t)\\
\end{split}
\end{equation*} 
Also $r(i,j)$ is decreasing with $j$ by FullNSC(k), 
for all $t \leq j < i-1 < k$:
\begin{equation*}
\begin{split}
      \bar{r}(i-1,j) \geq & \ \bar{r}(i-1,j-1)\\
      = & \sum_{\delta=1}^{(i-1) - (j-1)} r(i-1,\delta)/((i-1)-(j-1))\\
      = & \sum_{\delta=1}^{i-j} r(i-1,\delta)/(i-j)\\
      \geq = & \sum_{\delta=1}^{i-j} r(i,\delta)/(i-j)\\
      = & \bar{r}(i,j)
\end{split}
\end{equation*} 

Since $\geq$ is transitive we conclude:
$$\forall i\leq k, j < i : \bar{r}(i,j) \geq \bar{r}(k,j)$$
and consequently:
$$\forall i\leq k : \bar{r}(k,t+1) \leq \bar{r}(i,t)$$
therefore
\begin{equation}
\forall i \in t \ldots k : (i-t) \times imp^u_{r,p}(i) \leq 1/ (\bar{r}(i,t) \times p(0)) \leq 1 / (\bar{r}(k,t+1)\times p(0))
    \label{eq:bari}    
\end{equation}
For induction we assume lemma (\ref{lemma:main}) holds for all $steps_u(i,t): i < k$.
To prove it holds for $k$, we substitute 
$(i-t) \times imp^u_{r,p}(i)$ for $steps_u(i,t)$ in the definition of $steps_u(k,t)$. Since $steps_u(i,t)$ occurs positively in this definition, our inductive assumption ensures this yields an expression greater than $steps_u(k,t)$.
\begin{equation*}
\begin{split}
steps_u(k,t) & = imp^u_{r,p}(k)\times ( 1 + \sum_{i=0}^{k-1} (p(0) \times r(k,k-i) \times steps_u(i,t) ))\\
& =
imp^u_{r,p}(k)\times ( 1 + \sum_{i=t+1}^{k-1} (p(0) \times r(k,k-i) \times steps_u(i,t) ))\\
& \leq imp^u_{r,p}(k) \times ( 1 + \sum_{i=t+1}^{k-1} p(0) \times r(k,k-i)  \times (i-t) \times imp^u_{r,p}(i) )\\
& \leq imp^u_{r,p}(k) \times ( 1 + \sum_{i=t+1}^{k-1} \frac{p(0) \times r(k,k-i) }{ p(0) \times \bar{r}(k,t+1) }) \rm{\ \ by \ \ \ equation \ref{eq:bari}}\\
& = imp^u_{r,p}(k) \times ( 1 + \frac{p(0) \times (k-(t+1)) \times \bar{r}(k,t+1) }{  p(0) \times \bar{r}(k,t+1) }) \\
& = imp^u_{r,p}(k) \times ( 1 + k - (t+1)) \\
& \leq (k-t) \times imp^u_{r,p}(k)
\end{split}
\end{equation*}
\end{proof}

\subsection{Proof of lemma \ref{lemma:reducedsteps} {\bf Reduced Steps}}
\vspace{0.5cm}
\begin{lemma}[Reduced Steps]
\label{lemma:reducedsteps}
Assuming $\mathit{Full \ NSC(k)}$ 
$$steps_{r,p}(k,t) \leq steps^u_{r,p}(k,t)$$
\end{lemma}

Note that if for any $f, k, P, P_i : i \in 0..k-1$:\\
$$f(k,t) = (1/P) \times ( 1 + \sum_{i=0}^{k-1} P_i \times f(i,t))$$
where $\sum_{i=0}^{k-1} P_i =P$,
then 
\begin{equation*}
    \begin{split}
& P \times f(k,t) =  1 + \sum_{i=0}^{k-1} P_i \times f(i,t)\\
\therefore & \sum_{i=0}^{k-1} P_i \times f(k,t) -  \sum_{i=0}^{k-1} P_i \times f(i,t) = 1 \\
\therefore & \sum_{i=0}^{k-1} P_i \times (f(k,t)-f(i,t)) = 1     
    \end{split}
\end{equation*}
This holds for both $steps_{r,p}$ and for $steps^u_{r,p}$.

Consequently\\
$\sum_{i=0}^{k-1} r(k,k-i) \times p(0) \times (steps^u_{r,p}(k,t)-steps^u_{r,p}(i,t)) = 1$\\
and\\
$\sum_{i=0}^{k-1} pn(k,i) \times (steps_{r,p}(k,t)-steps_{r,p}(i,t)) = 1$\\

\begin{proof}
The proof is by induction on $k$.
For the base case, $k=t+1$, and
\begin{equation*}
    \begin{split}
steps_{r,p}(t+1,t) &= imp(p,t+1) = imp^u_{r,p}(p,t+1) = steps^u_{r,p}(t+1,t)\\
    \end{split}
\end{equation*}

For the inductive step, we assume \\
\begin{equation}
\forall i \leq k : steps_{r,p}(i,t) \leq steps^u_{r,p}(i,t)
\label{eq:stepseqone}
\end{equation}
Since $pn(k,i) \geq r(k,k-i) \times p$, and (by \ref{eq:stepseqone})\\
$steps_{r,p}(i,t) \leq steps^u_{r,p}(i,t)$ for each $i<k$,\\
and $steps^u_{r,p}(k,t) \geq steps^u_{r,p}(i,t)$ for each $i<k$\\
(which seems obvious, but requires the proof of lemma \ref{lemma:monotonesteps} below)\\
therefore
$steps_{r,p}(k,t) > steps^u_{r,p}(k,t)$
would imply\\
\begin{equation*}
    \begin{split}
&\sum_{i=0}^{k-1} pn(k,i) \times (steps_{r,p}(k,t)-steps_{r,p}(i,t)) \\
> &\sum_{i=0}^{k-1} r(k,k-i) \times p(0) \times (steps^u_{r,p}(k,t)-steps^u_{r,p}(i,t))        
    \end{split}
\end{equation*}
which is false because both are equal to $1$.\\
We conclude that
$$steps_{r,p}(k,t) \leq steps^u_{r,p}(k,t)$$
\end{proof}

\subsection{Proof of lemma \ref{lemma:monotonesteps} {\bf Monotony of steps}}

\begin{lemma}[Monotonicity of steps]
\label{lemma:monotonesteps}
The number of steps increases with distance from the target cost level.\\
If Assumptions 1,2,3 and 4
then
\begin{flalign*}
\ \ \ \forall k_1 \leq k_2 \leq k &: steps_u(k_1,t) \leq steps_u(k_2,t)&
\end{flalign*}
\end{lemma}
\begin{proof}
The proof is by induction on $k$.
The base case $steps_u(t+1,t) \geq steps_u(t,t)$ is immediate because $steps_u(t+1,t)$ is non-negative and $steps_u(t,t)=0$.

For the inductive case we can assume that
$\forall t \leq i<k : steps_u(i,t) \geq steps_u(i-1,t)$,
from which we infer
\begin{flalign}
\forall t \leq j<k-1 : steps_u(k-1,t) \geq steps_u(j,t)
\label{eq:indass}
\end{flalign}

The key step is a lemma:
\begin{lemma} Step Lemma\\
Assuming FullNSC(k) and equation \ref{eq:indass}:\\
$$1 + \sum_{\delta=2}^{k-(t+1)} : p \times r(k,\delta) \times steps_u(\delta,t) \geq steps_u(k-1,t) \times \sum_{\delta=2}^{k} : r(k,\delta) \times  p(k-\delta) $$
\label{lemma:stepsu}
\end{lemma}
This lemma is proven below.

Note, first, that the expression for $steps_u(k,t)$ is equivalent to:
\begin{flalign*}
steps_u(k,t) & = 1 + (1 - \sum_{i<k} : p(i) \times r(k,k-i)) \times steps_u((k,t)&\\
            & + \sum_{i=t+1}^{k-1} r(k,k-i) \times p \times steps_u(i,t)&
\end{flalign*}

\noindent Assume, for contradiction, that:
\begin{flalign}
steps_u(k,t) & < steps_u(k-1,t)&
\label{eq:asssteps2}
\end{flalign}

From this it will be inferred that:\\
$steps_u(k,t) \geq steps_u(k-1,t)$
which establishes the inductive step.

\begin{flalign*}
steps_u(k,t) = & 1 + (1-\sum_{\delta=1}^k r(k,\delta) \times p(k-\delta)) \times steps_u(k,t)
             + r(k,1) \times p(k-1) \times steps_u(k-1,t) &\\
             & + \sum_{\delta=2}^{k-(t+1)} r(k,\delta) \times p(k-\delta) \times steps_u(k-\delta,t)&\\
(\rm{by} \ \ref{eq:asssteps2}) \ \  > & 1 + (1-\sum_{\delta=1}^k r(k,\delta) \times p(k-\delta)) \times steps_u(k,t)
            + r(k,1) \times p(k-1) \times steps_u(k,t) &\\
            & + \sum_{\delta=2}^{k-(t+1)} r(k,\delta) \times p(k-\delta) \times steps_u(k-\delta,t)&\\
= & 1 + (1- \sum_{\delta=2}^k r(k,\delta) \times p(k-\delta)) \times steps_u(k,t) &\\
     & \ \ +  \sum_{\delta=2}^{k-(t+1)} r(k,\delta) \times p \times steps_u(k-\delta,t)&\\
\therefore 
steps_u(k,t) \times & \sum_{\delta=2}^k r(k,\delta) \times p(k-\delta) &\\
    > & 
        1 + \sum_{\delta=2}^{(k-(t+1)} r(k, \delta) \times p  \times steps_u(k-\delta,t)& \\
(\rm{by} \ \ lemma \ \ \ref{lemma:stepsu}) \ \  \geq & 
steps_u(k-1,t) \times \sum_{\delta=2}^{k} r(k,\delta) \times p(k-\delta)  &\\
\therefore 
steps_u(k,t) & > steps_u(k-1,t) & \\
\end{flalign*}
The assumption $steps_u(k)<steps_u(k-1)$ implies $steps_u(k) > steps_u(k-1)$, and we conclude by contradiction that $steps_u(k) \geq steps_u(k-1)$
\end{proof}
\subsection{Proof of lemma \ref{lemma:gandt} {\bf Steps upper bound}}
\vspace{0.5cm}
\begin{lemma}[Steps upper bound]
\label{lemma:gandt}
Assuming 
$\bar{r}(k,t) \geq p^<(t+1)/p(t)$
it follows that
$$imp^u_{r,p}(k) \times (k-t) \leq blind(t)$$
\end{lemma}
\begin{proof}

\begin{flalign*}
(k-t) \times & imp^u_{r,p}(k) \\
& = 1/ (p(t) \times \bar{r}(k,t)) \ \ {\rm by \ \ equation \ \ \ref{eq:impbarr}}\\
& \leq 1/ (p^<(t+1)\\
& = blind(t)
\end{flalign*}
\end{proof}

\end{appendix}
\end{document}